%% file: main.tex
\documentclass{article} 
\usepackage{iclr2027_conference,times}

\input{preamble}

\title{Latent Inference-Time Guidance of \\
Time Series Foundation Models}

\author{Chlo\'e Hashimoto-Cullen\\
LPSM, Sorbonne Universit\'e \\
\texttt{hashimoto@lspm.paris} 
\And
Amaury Durand \\
EDF R\&D\\
\And
Laurent Bozzi \\
EDF R\&D
\And
Benjamin Guedj \\
UCL \\
\And
Yannig Goude \\
EDF R\&D \\
\And
Sylvain Le Corff \\
LPSM, Sorbonne Universit\'e \\
}

\iclrfinalcopy 
\begin{document}

\maketitle

\begin{abstract}
Time Series Foundation Models (TSFMs) currently provide state-of-the-art results in forecasting tasks. They are available out-of-the-box and rely on in-context learning to make their predictions, which makes the quality of their performance highly sensitive to the user-selected lookback, covariates, horizon and training data distributions. In practise, the quality of the forecasts are variable but complementary, which highlights the need for a principled ensembling approach, rather than selecting the best context. This paper introduces Latent Inference-Time Guidance for TSFMs, which adaptively combines a pool of TSFM forecasts through a time-dependent latent space with independent components. The framework comes equipped with identifiability and reconstruction guarantees, whilst maintaining the off-the-shelf aspect of foundation models. We provide experiments on datasets at various frequencies and from multiple domains: these show that the approach is competitive with traditional ensembling approaches.
\end{abstract}

\section{Introduction}
\label{sec:introduction}
\input{sections/introduction}

\section{Background and Related Works}
\label{sec:background}
\input{sections/background}

\section{Latent Guidance of Foundation Models}
\label{sec:theory}
\input{sections/theory}

\section{Experiments}
\label{sec:experiments}
\input{sections/experiments}

\section{Discussion, Limitations and Conclusion}
\label{sec:discussion}
\input{sections/discussion}

\subsection*{AI use statement}
In this work, we used generative AI tools to implement methods, to clean and reformat datasets and to support qualitative and thematic data analysis.
We have not used generative AI tools for dataset generation, help in developing theoretical models or conceptual frameworks, formulating mathematical claims, providing critical ingredients for proving mathematical claims, assisting in the writing of proofs, proposing or refining hypotheses, designing or providing feedback on research methodology or experiments or interpreting results,
and assisting with translation is not applicable to this work.
Additionally, we used generative AI tools for creating and editing software code. We have reviewed all AI-assisted work: all LLM-generated code was verified and tested for correctness by the lead author. We take responsibility for the final content of this work, including text, claims or artifacts produced with the aid of generative AI.

\subsection*{Reproducibility statement}

The results given in \Cref{sec:theory} are derived from cited works. Assumptions made to apply the results are given in \Cref{sec:theory}. The proofs for further theoretical guarantees developed for \algo are given in Appendix~\ref{app:technical}. The code for \algo is linked to an anonymous Github in \Cref{sec:experiments}. Further empirical details on the empirical implementation are given in \Cref{app:experiments}.

\clearpage

\bibliography{references}
\bibliographystyle{iclr2027_conference}

\appendix
\section{Technical Results}
\label{app:technical}
\input{sections/app_technical}

\section{Empirical Details}
\label{app:experiments}
\input{sections/app_experiments}

\end{document}

%% file: preamble.tex
\usepackage[utf8]{inputenc} 
\usepackage[T1]{fontenc}    
\usepackage{hyperref}       
\usepackage{url}            
\usepackage{booktabs}       
\usepackage{amsfonts}       
\usepackage{nicefrac}       
\usepackage{microtype}      
\usepackage[table,xcdraw]{xcolor}         
\usepackage{xspace}

\usepackage{caption}
\usepackage{natbib}

\usepackage{amsmath}
\usepackage{amssymb}
\usepackage{amsthm}
\usepackage{mathtools}

\usepackage{cleveref}

\usepackage{booktabs}
\usepackage{csvsimple}
\usepackage{placeins}
\usepackage{multirow}

\makeatletter
\AddToHook{cmd/appendix/before}{\crefalias{section}{appendix}}
\makeatother

\def\eqsp{\;}
\def\rmd{\mathrm{d}}

\def\mcd{\mathcal{D}}
\def\mcl{\mathcal{L}}
\def\mcn{\mathcal{N}}
\def\mcx{\mathcal{X}}
\def\mcy{\mathcal{Y}}
\def\E{\mathbb{E}}
\def\R{\mathbb{R}}

\newtheorem{theorem}{Theorem}[section]
\newtheorem{proposition}[theorem]{Proposition}

\newcommand{\algo}{\textsc{Lit}i\textsc{g}-\textsc{Tsfm}\xspace}

\newcounter{hypA}
\newenvironment{hypA}{
    \refstepcounter{hypA}
    \begin{itemize}
    \item[{\bf A\arabic{hypA}}]
    }
{\end{itemize}}

\newcounter{hypB}
\newenvironment{hypB}{
    \refstepcounter{hypB}
    \begin{itemize}
    \item[{\bf B\arabic{hypB}}]
    }
{\end{itemize}}

\newcounter{hypC}

%% file: sections/introduction.tex
Time Series Foundation Models (TSFMs) have recently enabled competitive zero-shot and in-context forecasting across a wide range of domains. While deep learning methods for tabular and sequential data historically struggled to consistently outperform classical approaches \citep{shwartzziv2022tabular,grinsztajn2022why,mcelfresh2023when}, recent TSFM architectures based on Transformers \citep{nie2023a}, in-context learning \citep{lu2025incontext, liu2025timerxl} and large-scale pre-training have narrowed this gap \citep{liu2024timer}. A wealth of these models are currently available \citep{grinsztajn_tabpfn-25_2025, ansari_chronos-2_2025, liu_moirai_2025, auer_tirex_2025, qu_tabiclv2_2026}, with benchmarks such as GIFT-Eval, fev-bench and Chronos available to compare them \citep{aksu2024gifteval, shchur2025fev, ansari_chronos_2024}. These models offer a compelling paradigm shift: rather than training a task-specific model, one can leverage a pre-trained model and adapt it to a chosen context at inference time.

Despite these advances, foundation models -- and, more specifically, TSFMs -- are sensitive to their input. For a given model, the prompt (which in the case of TSFMs encompasses the covariates, lookback and forecasting horizon choices) greatly affects the model accuracy \citep{michalkiewicz2025not, romanou2026brittlebench}. In short, different pre-trained TSFMs, or even different configurations of the same model, often exhibit complementary strengths across datasets and temporal regimes, due to their differences in architectures and pre-training data \citep{meyer2025benchmarking, berthelier2026investigating}. This variability suggests that the central challenge is not selecting a single model, but rather combining multiple forecasts in a principled and adaptive manner.

Classical approaches to forecasting with ensembles \citep{oliveira2015ensembles}, such as mixtures of experts and online aggregation methods \citep{gaillard2014second, wintenberger2017optimal}, as well as state-space techniques such as Kalman filtering \citep{vilmarest2024adaptive}, provide well-established solutions to this mixture of models problem. However, existing methods are typically restricted to linear aggregation schemes and do not exploit the representational structure of modern foundation models. Conversely, a foundation model can be fine-tuned on a specific dataset; but this is computationally expensive, statistically inefficient in low-data regimes, and can be impossible when the model is only accessible as a black box.

This motivates a different perspective, which we refer to as \emph{inference-time guidance}. Combination mechanisms exist in generative modelling for time series and other modalities \citep{fedus2022switch,liu_moirai_2025,shi2025timemoe,yang2026expert}. These approaches combine expert models within a larger model. Rather than modifying or combining the existing foundation models themselves, we learn a lightweight mechanism that steers or combines the expert outputs using downstream observations. This paradigm has proved effective in generative modelling for time series and other modalities \citep{lee2025lightweight,jiang2025sra, bender2026visual}, where external signals can guide pre-trained models without retraining. Beyond time series, approaches guiding existing foundation models exist, mostly using autoencoding architectures to identify relevant features at inference \citep{bricken2023towards,oneill2024towards,le2024learning}. However, to the best of our knowledge, and despite its conceptual appeal, there is no general probabilistic framework for inference-time latent guidance in time-series forecasting.

In this paper, we treat TSFM forecasts as a collection of experts, and use the combination of variational latent modelling with a structured state-space formulation to propose a principled probabilistic framework for combining TSFM forecasts, while capturing temporal dependencies and uncertainty. The resulting model introduces a time-dependent latent state that encodes how each expert should be weighted or corrected at each time step. In its simplest form, this reduces to a dynamic mixture of experts with time-varying weights; more generally, it defines a flexible nonlinear guidance mechanism acting on the forecasts. Importantly, our method leaves the underlying foundation models frozen: only the lightweight guidance mechanism is learned, without requiring access to or fine-tuning of the TSFM parameters. A key component of this approach is the structure given to the latent space, where rather than learning arbitrary hidden representations, the representations learned have an independent component decomposition, which connects the approach to work in nonlinear independent component analysis (ICA), where auxiliary variables such as temporal dependence and covariates enables identifiability of latent representations \citep{halva_disentangling_2021}. In our setting, this allows guarantees on the induced denoised process.

Our contributions are threefold. First, we introduce a probabilistic latent guidance framework that adaptively combines and corrects forecasts from frozen TSFMs through a time-dependent latent state. Second, we connect this construction to structured nonlinear ICA and establish conditions for identifiability of the induced denoised process together with a reconstruction stability guarantee for the variational smoother.  Third, we demonstrate the empirical benefits of this approach across datasets spanning multiple domains and sampling frequencies.

The remainder of the paper is organised as follows. In \Cref{sec:background}, we introduce the lines of research at the intersection of which this work lies. \Cref{sec:theory} introduces the proposed latent guidance model, its variational formulation, and its inference and forecasting procedures, as well as presenting identifiability analysis and theoretical guarantees. Experiments results are reported in \Cref{sec:experiments}, where we benchmark our proposed method with time series foundation models and classical aggregation methods. \Cref{sec:discussion} discusses the avenues for future work. Additional technical details and proofs to our results are gathered in \Cref{app:technical} and further experimental details are given in \Cref{app:experiments}.

%% file: sections/background.tex
\textbf{Foundation models. } Prior to 2023, a variety of deep learning architectures existed for tabular data and time series. However, these struggled to work better than smaller-scale statistical and machine learning methods \citep{grinsztajn2022why,mcelfresh2023when}. With \citet{nie2023a}, the transformer architecture which had been introduced for computer vision and text-based tasks was adapted to tabular data, providing a deep learning architecture which was competitive with traditional methods. This was enabled the development of Tabular Foundation Models (TFMs), whose encoder or decoder architectures were pre-trained on large amounts of data. Encoder architectures mask parts of the time series and learn to reconstruct the missing segments \citep{woo2024unified, liu_moirai_2025}, whereas decoder architectures depend on causal modelling \citep{das2024decoder}. The first generation of TFMs were univariate, which limited their forecasting capacities. \citet{hollmann2023tabpfn, qu2025tabicl} provided the TabPFN and TabICL models, which could learn tabular patterns in-context, allowing covariate data and multivariate forecasting. Alongside these in-context learning architectures, \citet{ansari_chronos-2_2025} released Chronos-2, an encoder architecture that takes context covariates and \citet{das2024decoder} have recently release TimesFM 3, which is the latest generation of a patch-based, transformer encoder architecture. Even in these latest generation of models, variability in performance can be observed due to the difference in synthetic data used at training, but these models are now being used in many applications. Multiple TSFMs now combine multiple outputs to address this variability \citep{liu2025moiraimoe}.
Benchmarks now exist to compare foundation models' performance \citep{aksu2024gifteval, shchur2025fev}, with new models coming out on an almost weekly basis. 

TabICL is trained with in-context learning. To train the models, consider covariates $\mathbf{x}$ in input space $\mcx \subseteq \R^m$ (with $m>1$) and target variables $y$ in output space $\mcy \subseteq \R$, from which $n$ training points are sampled: $\mcd_\textrm{train} = (\mathbf{x}_{\textrm{train}}^i,y_{\textrm{train}}^i)_{i=1}^n \sim p(\mcd)$. Test points are sampled from the same distribution: $(\mathbf{x}_\textrm{test}, y_\textrm{test}) \sim p(\mcd)$. Thus, predictions can be made with a single forward pass of a pre-trained neural network: \(y_\textrm{test} \sim p(\cdot\mid \mathbf{x}_\textrm{test}, \mcd_\textrm{train})\). During training, the neural network parameters are updated with stochastic gradient descent. The models are trained with synthetic data (referred to as the \emph{prior} in the literature), which is generated using 
additional preprocessing. Inference is then run on real-world datasets. The model now has variants which have been trained on time series following the same rationale, but can be used directly on time series in tabular format. We refer to both types as a TSFM.

\textbf{Variational AutoEncoders. } Variational Auto-Encoders (VAE) introduce approximations of a target conditional distribution in the context of latent data models, see \cite{rezende2014stochastic,kingma2019introduction}. 
Consider target data \(\mathbf{y}_{1:T} \in\R^{T\times m}\) with covariates \(\mathbf{x}_{1:T} \in\R^{T\times d}\) and latent representation \(\mathbf{z}_{1:T} \in\R^{T\times d_\ell}\).
The latent variable generative model defines a joint density $(\mathbf{z}_{1:T},\mathbf{y}_{1:T})\mapsto p_\theta(\mathbf{y}_{1:T}| \mathbf{z}_{1:T},\mathbf{x}_{1:T})$   by specifying a prior $\mathbf{z}_{1:T}\mapsto p_{\theta}(\mathbf{z}_{1:T}|\mathbf{x}_{1:t})$ over the latent variable $\mathbf{z}_{1:T}$ and a conditional density $\mathbf{y}_{1:T}\mapsto p_\theta(\mathbf{y}_{1:T}|\mathbf{z}_{1:T},\mathbf{x}_{1:T})$ (also referred to as the decoder).  The normalised log-likelihood of $(\mathbf{y}^i_{1:T})_{1\leq i \leq n}$ is therefore given by
\[
\ell_n(\theta) = \frac{1}{n}\sum_{i=1}^n \log p_{\theta}(\mathbf{y}^i_{1:T}|\mathbf{x}_{1:T}) = \frac{1}{n}\sum_{i=1}^n \log \int p_{\theta}(\mathbf{z}_{1:T}|\mathbf{x}_{1:T})p_{\theta}(\mathbf{y}^i_{1:T}|\mathbf{z}_{1:T},\mathbf{x}_{1:T})\rmd \mathbf{z}_{1:T}\eqsp,
\]
and the conditional distribution $p_{\theta}(\mathbf{z}_{1:T}|\mathbf{y}_{1:T},\mathbf{x}_{1:T})\propto p_{\theta}(\mathbf{z}_{1:T}|\mathbf{x}_{1:T})p_{\theta}(\mathbf{y}_{1:T}|\mathbf{z}_{1:T},\mathbf{x}_{1:T})$. Since the integral for the marginalising the latent variable is intractable, the marginal likelihood functions $p_{\theta}(\mathbf{y}^i_{1:T})$ for $1\leq i \leq n$ are not available explicitly. So it is not possible to maximise the average marginal log-likelihood of the data. Since a maximum likelihood estimator cannot be computed simply, VAEs introduce a variational approach which aims to simultaneously provide a parameter estimate and an approximation of the conditional distribution of the latent variable given the observation. Consider a family of probability density functions $\{ q_{\varphi}(\cdot|\mathbf{y}_{1:T},\mathbf{x}_{1:T})\}_{\varphi\in\Phi}$. 
Then $\log p_\theta(\mathbf{y}_{1:T}|\mathbf{x}_{1:T})\geq \mathcal{L}(\theta,\varphi,\mathbf{y}_{1:T},\mathbf{x}_{1:T})$, where
\begin{equation}
\label{eq:elbo}
    \mathcal{L}(\theta,\varphi,\mathbf{y}_{1:T},\mathbf{x}_{1:T}) = \mathbb{E}_{q_{\varphi}(\cdot|\mathbf{y}_{1:T},\mathbf{x}_{1:T})}\left[\log \frac{p_\theta(\mathbf{z}_{1:T},\mathbf{y}_{1:T}|\mathbf{x}_{1:T})}{q_{\varphi}(\mathbf{z}_{1:T}|\mathbf{y}_{1:T},\mathbf{x}_{1:T})}\right]
\end{equation}
is the Evidence Lower BOund (ELBO). In this setting, \(q_\varphi(\mathbf{z}_{1:T}|\mathbf{y}_{1:T}, \mathbf{x}_{1:T})\) is referred to as the \emph{encoder}.

Variational excess-risk bounds exist for state-space models, and can be used to bound the risk. \citet{chagneux2024importance} use a backward factorisation of the variational distribution to provide bounds on the estimation error. Within the VAE literature, \citet{cherief-abdellatif_pac-bayesian_2022} provide PAC-Bayesian error bounds on the reconstruction capacities of VAEs. \citet{mbacke_statistical_2023} develop conditional bounds which can then be applied to reconstruction, generation and re-generation. \citet{hashimotocullen2026pacbayesian} extend these conditional bounds to sequential data with a Markovian latent structure for reconstruction errors. 

\textbf{Structured Nonlinear Independent Component Analysis (ICA).} Nonlinear ICA \citep{hyvarinen1999nonlinear} assumes that observed data $\mathbf{x}\in \R^n$ is generated by some invertible nonlinear mixing function $f$ from $\mathbf{s}\in\R^d$ latent independent components: \(\mathbf{x} = f(\mathbf{s})\), where $p(\mathbf{s}) = \prod_{i=1}^dp(\mathbf{s}^{(i)})$. \citet{halva_disentangling_2021} make additional assumptions on stationarity (for any $t$, $t^\prime$, then $(\mathbf{s}^{(i)}_t)_{1\leq i\leq d}$ and $(\mathbf{s}^{(i)}_{t^{\prime}})_{1\leq i\leq d}$ are the same), conditional component independence, the nonlinear mixing function's injectivity and i.i.d. noise at each timestep, to propose Structured Nonlinear ICA (SNICA), where for independent latent components $\mathbf{s}_t = (\mathbf{s}_t^{(1)}, \ldots, \mathbf{s}_t^{(d)})$, an observed $\mathbf{x}_t = f(\mathbf{s}_t)+\varepsilon_t$. Works proposing a structured VAEs which learns following a nonlinear ICA framework exist \citep{khemakhem2020variational, brehmer2022weakly}. \citet{connor2021variational} propose a VAE architecture with a learned latent structure.

%% file: sections/theory.tex
Recent advances in nonlinear ICA provide a principled framework for learning structured and identifiable latent representations in high-dimensional settings. In contrast to classical latent-variable models, nonlinear ICA leverages auxiliary structure (specifically for our setting; temporal dependencies and covariates) to ensure identifiability of the latent components. Each latent component exhibits its own dependency structure, while preserving conditional independence across components. 

We introduce a nonlinear ICA perspective to introduce structured latent variables as a lightweight probabilistic interface for guiding pretrained foundation models. We call this framework \textbf{L}atent \textbf{I}nference-\textbf{Ti}me \textbf{G}uidance for \textbf{TSFM}s (\algo). The model is designed to operate at inference time without retraining the underlying foundation models. By modelling latent variables as time-dependent processes, the framework naturally captures sequential dependencies and allows for adaptive conditioning based on past observations or covariates.

This formulation is motivated by recent advances in nonlinear ICA, such as identifiable constructions \cite[$\Delta$-SNICA]{halva_disentangling_2021} and by the growing use of foundation models, for which lightweight probabilistic structures are needed to enable principled conditioning and guidance without retraining large-scale generative systems.

\textbf{Notation.}
Consider $\mathbf{y}_{1:T} \in\R^{T\times m}$ a sequence of observations. Let $\mathbf{x}_{1:T} \in\R^{T\times d}$ be a sequence of covariates and $ \mathbf{F} = (\mathbf{F}_{1:T}^j)_{1\leq j \leq N} \in \R^{N\times T\times m}$ be a set of outputs of pre-trained foundation models. For all $t\geq 1$, we write $\mathbf{F}_t =F(\mathbf{x}_{1:t})$ the vector containing the predictions given by all foundation models at time $t$. We also write $\mathcal{N}(\cdot;\mu,\Sigma)$ for the Gaussian probability density function with mean $\mu$ and variance $\Sigma$.
 
Consider the following nonlinear ICA model \algo
\begin{align*}
y_t &= h_\theta(\mathbf{F}_t, z_t) + \varepsilon_t   \in \mathbb{R}^m \\
z_{t+1,i} &= f_\theta(x_{t+1},z_{t,i}) +  \eta_t \in \mathbb{R}\eqsp,
\end{align*}
with $z_t = (z_{t,i})_{1\leq i \leq p}, z_t\in\R^{d_\ell}$ the independent latent components, $(\varepsilon_t)_{t\geq 1}$ and $(\eta_t)_{t\geq 1}$ are i.i.d. with  $\varepsilon_t \sim \mathcal{N}(0,\Sigma)$ and $\eta_t \sim \mathcal{N}(0, \rho^2)$. Since \algo is a nonlinear ICA model in which latent dynamics and observations are jointly parametrised and learned, its transition kernel $p_\theta(z_t \mid z_{t-1}, x_t)$ is modelled through a parametric map $f_\theta$ and the emission distribution $p_\theta(y_t \mid z_t)$ through a \emph{decoder} $h_\theta$.

The latent prior factorises as
\[
p_\theta(\mathbf{z}_{1:T}|\mathbf{x}_{1:T}) = p_\theta(z_1)\prod_{t=2}^T\prod_{i=1}^p p_\theta(z_{t,i} \mid z_{t-1,i},x_t),
\]
with Gaussian parameterisations $p_\theta(z_1) = \mathcal{N}(z_1;\mu_1, \rho_1^2)$ and $p_\theta(z_{t,i} \mid z_{t-1,i}) = \mathcal{N}(z_{t,i};f_\theta(x_t, z_{t-1,i}), \rho^2)$ for each latent component $p$ and time step $t$. The observation model satisfies
\[
p_\theta(\mathbf{y}_{1:T} \mid \mathbf{z}_{1:T}) = \prod_{t=1}^T p_\theta(y_t \mid z_t), 
\quad 
p_\theta(y_t \mid z_t) = \mathcal{N}(y_t;h_\theta(\mathbf{F}_t, z_t), \Sigma).
\]

To approximate the unknown posterior distribution, we consider a structured variational family:
\[
q_\varphi(\mathbf{z}_{1:T} \mid \mathbf{y}_{1:T}, \mathbf{x}_{1:T})
= q_{\varphi,T}(z_T \mid \mathbf{y}_{1:T}, \mathbf{x}_{1:T})\prod_{t=1}^{T-1} q_{\varphi,t|t+1}(z_{t} \mid z_{t+1}, \mathbf{y}_{1:t}, \mathbf{x}_{1:t})\eqsp.
\]
A standard setting assumes that all conditional distributions are Gaussian with neural-network parametrisations. For instance, we may choose 
\begin{align}
\label{eq:enc}
    q_{\varphi,T}(z_T \mid \mathbf{y}_{1:T}, \mathbf{x}_{1:T}) & = \mcn(z_T;\mu_\varphi(\mathbf{y}_{1:T}, \mathbf{x}_{1:T}),\Sigma_\varphi(\mathbf{y}_{1:T}, \mathbf{x}_{1:T})) \\
    q_{\varphi,t|t+1}(z_{t}\mid z_{t+1}, \mathbf{x}_{1:t}, \mathbf{y}_{1:t}) &= \mcn(z_{t};\mu_{\varphi, t}(z_{t+1},\mathbf{x}_{1:t}, \mathbf{y}_{1:t}),\Sigma_{\varphi, t}(z_{t+1},\mathbf{x}_{1:t}, \mathbf{y}_{1:t}))\eqsp, \notag
\end{align}
where $\mu_{\varphi}$, $\mu_{\varphi, t}$, $\Sigma_{\varphi}$, $\Sigma_{\varphi, t}$ are neural networks chosen depending on the applications. 

The model is trained by maximising the ELBO given by \eqref{eq:elbo} derived from the above variational decomposition. Rather than relying on its explicit closed form, we emphasise that the ELBO admits a decomposition into three interpretable contributions: a reconstruction term associated with the observation model, a dynamical consistency term induced by the latent transition, and a regularisation term corresponding to the divergence between the variational posterior and the prior. This allows to derive an explicit loss function: the derivation is detailed in Appendix~\ref{subapp:elbo}.

\textbf{Inference and forecasting.}
Once the model is trained, prediction and sampling rely on propagating the learned latent dynamics under the generative model $p_\theta$, optionally combined with variational approximations. The predictive distribution is given by
\[
p_\theta(y_{t+1} \mid \mathbf{y}_{1:t}) 
= \int p_\theta(y_{t+1} \mid z_{t+1}) \, p_\theta(z_{t+1} \mid \mathbf{y}_{1:t}) \, \mathrm{d}z_{t+1},
\]
where
\[
p_\theta(z_{t+1} \mid \mathbf{y}_{1:t}) 
= \int p_\theta(z_{t+1} \mid z_t, x_{t+1}) \, p_\theta(z_t \mid \mathbf{y}_{1:t}) \, \mathrm{d}z_t.
\]
A practical approximation consists in initialising $z_t \sim q_\varphi(\cdot \mid \mathbf{y}_{1:t},\mathbf{x}_{1:t})$, and propagating forward using the generative dynamics. To improve forecasting without retraining $p_\theta$, one can refine the variational approximation by using  a VAMP-type adaptation of the posterior at inference time \citep{tomczak_vae_2018}. Further details are given in Appendix~\ref{subapp:forecasting}.

\textbf{Identifiability results. }
We now discuss identifiability of the latent guidance mechanism. Recall that, conditionally on the
outputs of the foundation models, the observation equation of \algo can be written as
\[
    y_t = s_t + \varepsilon_t,
    \qquad \textrm{where} \;
    s_t = h_\theta(\mathbf{F}_t,z_t)\eqsp.
\]

In the following, all statements are made conditionally on $\sigma(\mathbf{F_t},\ t\geq 1)$.

\begin{hypA}
    \label{assum:bounded:F}
    $(\mathbf{F}_t)_{t\geq 1}$ are assumed to be independent of
$(z_t,\varepsilon_t)_{t\geq 1}$, and bounded.
\end{hypA}

The identifiability question is twofold. First, we ask whether the denoised signal
\((s_t)_{t\geq 1}\) is identifiable from the noisy observations \((y_t)_{t\geq 1}\).
Then, we ask whether the latent ICA components \(z_t\) are identifiable from \(s_t\) and
\(\mathbf{F}_t\). The first question follows from the noisy Structured Nonlinear ICA argument of
\cite{halva_disentangling_2021}, while the second requires additional conditions on the decoder
\((\mathbf{F}_t,z_t)\mapsto h_\theta(\mathbf{F}_t,z_t)\). Following \cite{halva_disentangling_2021}, we consider the conditions on the statistical properties of the signal for some $t_2>t_1\geq 1$.

\begin{hypA}
\label{assum:light:tail}
For some $\rho<3$, there exist constants $A,B>0$ such that, for all $\lambda\in\mathbb R^m$,
$
\mathbb{E}\left[\exp(\langle \lambda,s_{t_1}\rangle)\right] \leq A\exp(B\|\lambda\|^\rho)\eqsp.
$
\end{hypA}

\begin{hypA}
\label{assum:non:degenerate}
For all $\eta\in\mathbb C^m$,  the random variable $\mathbb E[\exp(\langle \eta,s_{t_2}\rangle)\mid s_{t_1}]$ is not almost surely equal to zero.
\end{hypA}

\begin{hypA}
\label{assum:non:gaussian}
There does not exist $\eta\in\mathbb R^m$ and independent random variables $\tilde s,u$ such that $u$ is a non-degenerate Gaussian random variable and
$
\langle \eta,s_{t_1}\rangle \overset{d}{=} \tilde s+u\eqsp.
$
\end{hypA}

\begin{proposition}[\citealp{halva_disentangling_2021}]
\label{prop:ident:signal}
Assume that A\ref{assum:bounded:F}--A\ref{assum:non:gaussian} hold. Then, up to translation, for all $k>2$ and all
$(t_3,\ldots,t_k)$, the map associating the distribution of $(s_{t_1},\ldots,s_{t_k})$
to the distribution of $(y_{t_1},\ldots,y_{t_k})$
is one-to-one. In particular, the distribution of the denoised latent guidance signal is identifiable
from the noisy observations.

\end{proposition}

We now derive a reconstruction guarantee. Following \cite{JMLR:v25:22-1392},  our proposed structured ICA model provides reconstruction guarantees within a state-space modeling framework. In this setting,  under a mixing-model assumption and assuming that the learned variational distribution is close in total variation distance to the true conditional distribution,  yields a reconstruction risk that grows at most linearly with the time horizon. 
Let $\phi^\theta_{1:T}$ denote the true smoothing distribution of $\mathbf{z}_{1:T}$ given
$(\mathbf{y}_{1:T},\mathbf{x}_{1:T})$ and write $\phi^\theta_{t}$  the true filtering distribution of $z_{t}$ given
$(\mathbf{y}_{1:t},\mathbf{x}_{1:t})$ for $1\leq t\leq T$. For all $2\leq t \leq T$ and let $b^\theta_{t-1|t}(\cdot\mid z_{t+1}, \mathbf{x}_{1:t}, \mathbf{y}_{1:t})$ be the backward Markov kernel at time $t$:
$$
b^\theta_{t-1|t}(z_{t-1}\mid z_{t}, \mathbf{x}_{1:t}, \mathbf{y}_{1:t}) \propto p_\theta(z_{t}|z_{t-1},x_{t})\phi^\theta_{t-1}(z_{t-1}|\mathbf{y}_{1:t-1}, \mathbf{x}_{1:t-1})\eqsp.
$$
Therefore, 
$\phi^\theta_{1:T}=\phi^\theta_T\prod_{t=1}^{T-1}b^\theta_{t\mid t+1}$, which matches
the factorisation of the variational family
$q_\varphi=q_{\varphi,T}\prod_{t=1}^{T-1}q_{\varphi,t\mid t+1}$.
Consider the following  assumptions.

\begin{hypB}
\label{assum:bounded:h:obs}
There exist constants $c_1, c_2$ such that for all $t\geq 1$, $\theta$, $\mathbf{F}_t$, $z_t$,
$$
\|h_\theta(\mathbf{F}_t,z_t)\|\leq c_1, \quad |y_t|\leq c_2\eqsp.
$$
\end{hypB}

\begin{hypB}
\label{assum:strong:mixing}
The proposed state-space model and the variational density 
satisfy a uniform mixing condition. There exist $0<\sigma_-<\sigma_+<\infty$ such that for all $1\leq t \leq T-1$, $\theta$, $\varphi$, $\mathbf{z}_{1:T}$, $\mathbf{x}_{1:T}$, $\mathbf{y}_{1:T}$,
$$
\sigma_-\leq p_\theta(z_{t+1}|z_t,x_{t+1})p_\theta(y_{t+1}|z_{t+1}) \leq \sigma_+ \; \mathrm{and}\;  \sigma_-\leq  q_{\varphi,t|t+1}(z_{t}\mid z_{t+1}, \mathbf{x}_{1:t}, \mathbf{y}_{1:t}) \leq \sigma_+\eqsp.
$$
\end{hypB}

Assumption B\ref{assum:strong:mixing} is common in the nonlinear state space literature to obtain quantitative bounds for the approximation and control of joint smoothing distributions. This assumption does not hold on an unbounded state space. It can usually be satisfied by considering truncated latent spaces. In addition, note that
compactly supported innovation law ensure A\ref{assum:light:tail} with
$\rho<2$.

\begin{hypB}
\label{assum:variational:precision}
There exists $\varepsilon>0$ such that
$$
\|q_{\varphi,T}(\cdot | \mathbf{y}_{1:T}, \mathbf{x}_{1:T})-\phi^\theta_{T}(\cdot|\mathbf{y}_{1:T}, \mathbf{x}_{1:T})\|_{\mathrm{TV}}\leq \varepsilon,
$$
and, for all $2\leq t\leq T$ and all $z_t$,
$$
\|q^\phi_{t-1|t}(\cdot \mid z_{t}, \mathbf{x}_{1:t}, \mathbf{y}_{1:t})-b^\theta_{t-1|t}(\cdot\mid z_{t}, \mathbf{x}_{1:t}, \mathbf{y}_{1:t})\|_{\mathrm{TV}} \leq \varepsilon,
$$
where \(b^\theta_{t-1|t}\) is the backward smoothing kernel.
\end{hypB}

Proposition~\ref{prop:reconstruction} turns a local variational approximation guarantee into a global smoothing guarantee for our latent mixture of foundation models. If the backward variational kernel are uniformly close in total variation to their exact smoothing counterparts, with local error at most $\varepsilon$, then the error on additive smoothing functionals, and classical reconstruction errors, grow at most linearly with the time horizon. 
\begin{proposition}
\label{prop:reconstruction}
Assume that B1--B3 hold. Let
\[
    \ell_t(z_t)=\|y_t-h_\theta(\mathbf{F}_t,z_t)\|^2.
\]
Then there exists a constant $C>0$ such that
\[
    \left|
        \frac1T\sum_{t=1}^T
        \mathbb E_{q_{\varphi}(\cdot | \mathbf{y}_{1:T}, \mathbf{x}_{1:T})}[\ell_t(z_t)]
        -
        \frac1T\sum_{t=1}^T
        \mathbb E_{\phi^\theta_{1:T}(\cdot | \mathbf{y}_{1:T}, \mathbf{x}_{1:T})}[\ell_t(z_t)]
    \right|
    \leq
    C\varepsilon\eqsp.
\]
The excess reconstruction risk of the variational smoother with respect to
the oracle smoother is controlled linearly by the approximation error of the backward variational kernels.
\end{proposition}
\begin{proof}
    The proof is deferred to \Cref{subapp:pf_reconstruction}.
\end{proof}
Proposition~\ref{prop:reconstruction} bounds the gap between the variational
reconstruction risk and the risk of the oracle smoother. The linear
growth in $T$ of the unnormalised bound is important: it shows that structured backward variational families can preserve the stability properties of exact smoothing and that local approximation errors do not compound exponentially. 

%% file: sections/experiments.tex
To evaluate the performance of our proposed framework \algo, we implement it on datasets covering multiple domains, frequencies and levels of noise.
All experiments are implemented in PyTorch and run on a MacBook Air with an M2 chip and 16GB RAM. Non-deterministic methods are each run ten times, and results are given as mean (std). We report compute cost for our method compared to its baselines for each dataset in \Cref{app:experiments}. The code is available on GitHub\footnote{\url{https://anonymous.4open.science/r/LITiG-TSFM-77FD/README.md}}.

\textbf{Metrics.} To evaluate the reconstruction and forecasting tasks, we use the Root Mean Square Error and the Mean Absolute Error. Their expressions are given in \Cref{app:experiments}. 

\textbf{Datasets. } To evaluate the \algo framework, we evaluate its forecasting capacities on datasets from multiple domains and at a variety of frequencies. This provides a wealth of model behaviours with varying levels of noise and seasonality, and from hourly to monthly timesteps. Further information on these datasets and how they are preprocessed is available in Appendix~\ref{app:experiments}.
\begin{table}[ht]
\centering
\small
\begin{tabular}{
>{\columncolor[HTML]{FFFFFF}}c 
>{\columncolor[HTML]{FFFFFF}}c 
>{\columncolor[HTML]{FFFFFF}}c 
>{\columncolor[HTML]{FFFFFF}}c 
>{\columncolor[HTML]{FFFFFF}}c 
>{\columncolor[HTML]{FFFFFF}}c }
\textbf{Dataset}         & \textbf{Domain} & \textbf{Freq} & \textbf{Length} & \textbf{\# Covariates} & \textbf{\# Series} \\ \hline
London Smart Meter \citep{SmartMeter_dataset}       & Energy          & D              & 606             & 17                            & 5                         \\ \hline
Mauna Loa \citep{maunaloa_dataset}                & Climate         & M            & 627             & 9                             & 1                         \\ \hline
NN5 Weekly \citep{godahewa2020nn5}               & Finance         & W             & 837             & 4                             & 111                       \\ \hline
Washington Bicycle Share \citep{Bike_dataset} & Traffic         & H             & 17378           & 30                            & 1                        
\end{tabular}
\caption{Datasets used to evaluate \algo and their key characteristics.}
\label{tab:datasets}
\end{table}

\subsection{Baselines}
\subsubsection{Ideal Settings Baselines}
\textbf{Highest-performing TSFM context.} We prompt TabICLv2, a foundation model with strong performances in the GIFT-Eval \citep{aksu2024gifteval} and fev-bench \citep{shchur2025fev} baselines. Rather than looking to find the best TSFM and its best configuration, we choose TabICLv2 as a representative sample of the current state of the art for TSFMs and use a set of its forecasts. 
To replicate implementation conditions where performance is evaluated on a static training dataset to select a model, the best TSFM from the validation fold of the data is retained as the best TSFM, and evaluated on the test fold of the data. Further details on how the TSFMs are prompted for each dataset are given in Appendix~\ref{app:experiments}. The group of high-performing forecasts is called \emph{experts} for the rest of this work, and is the same for each of the methods implemented.

\textbf{Oracle mixture.} Consider a set of \(N\) experts \(\{\mathbf{F}_{1:T}^k\}_{k=1}^N\) for a time series \(\mathbf{y}_{1:T}\). At each time step \(1\leq t\leq T\), these experts have loss \(p_{k,t}\). The oracle forecast of these experts is given by \(\hat{y}_{t} = F_{\hat{k},t}\), where \(\hat{k} = \underset{k}{\arg\min}\, p_{k,t}\) for \( p_{k,t}\) the loss of an expert \(\mathbf{F}^k_{1:T}\) at time step \(t\). This benchmark is inaccessible at inference, since it requires having access to the ground truth to select the best expert at each time step. So we use it as a performance floor, rather than a baseline to beat.

\subsubsection{Mixture of Experts Baselines}
\textbf{Foundation models.} We use TabICLv2 raw quantiles and mean forecasts as the experts which will be aggregated. The mean and standard deviation of these forecasts is reported (under the label \emph{All TSFMs}), with the same context given for raw quantile and mean outputs. More detail on the contexts given to the TSFM is available in Appendix~\ref{app:experiments}.

\textbf{Mixture of experts.} In time series forecasting, principled mixtures of high-experts is a traditional method to improve forecasts with complementary strengths. For a set of forecasts $\mathbf{F}_t = \{F_{k,t}\}_{k=1}^K$, each with weight $\pi_{k,t}$ depending on historical performance and with $\sum_{k=1}^K \pi_{k,t}=1$, the mixture of experts prediction is given by $y_t=\sum_{k=1}^K\pi_{k,t}F_{k,t}$. The {ML-Poly} algorithm \citep{gaillard2014second} from the Opera package\footnote{\url{https://github.com/Dralliag/opera-python}} is implemented on the experts obtained by prompting TabICLv2. The ML-Poly algorithm differs from mixture-of-expert TSFMs. Indeed, the weights are learned in a linear way, without additional neural network architectures \citep{liu2025moiraimoe}; also, they do not necessitate agentic evaluation \citep{cao_conversational_2025}. This approach comes with robustness guarantees to drifts in the data distribution, as does the oracle mixture. This makes ML-Poly an attractive method for time series forecasting, where distribution drift is a common challenge.

\textbf{Kalman Filter.} A Kalman Filter updates the weights of a state space model following the equations 
\[
\mathbf{y}_{t+1:t+H} = \alpha_t \mathbf{F}_t + \beta_t + \varepsilon_t , \quad\textrm{where}\quad
    \begin{pmatrix}
        \alpha_{t+1} \\
        \beta_{t+1}
    \end{pmatrix} = A\begin{pmatrix}
        \alpha_t \\
        \beta_t
    \end{pmatrix} + \eta_t,
\]
$\alpha_t$ is the weights for the model output, $\beta_t$ is the bias of the model, 
$\varepsilon \sim \mcn(0, Q)$ and $\eta_t\sim\mcn(0, R)$. In our experiment, we fix $A={\rm Id}$ and learn $Q, R$ using an  expectation maximisation implemented in the  viking\_kalman package\footnote{\url{https://github.com/EDF-Lab/viking_kalman}}.

\subsection{Results}
For each dataset, we train \algo with the objective derived in \Cref{app:technical}. Details of the implemented neural networks are given in \Cref{app:experiments}. The encoder, decoder and prior architectures used are flexible and can be adapted to the specificities of each dataset.

\textbf{Forecasting. }
We provide results for the London Smart Meter and NN5 datasets in \Cref{fig:strong_tsfm_results}. \Cref{fig:weak_tsfm_results} provides results for the Washington Bicycle Share and Mauna Loa datasets. In both figures, the ML-Poly and Kalman Filter mean and standard deviation are reported for datasets with multiple time series; variance is not available for these models on single-series datasets, since the methods are deterministic. The oracle and best TSFM performances given for the London Smart Meter and NN5 datasets are given as the average oracle and best TSFM metrics over the series in the dataset. \Cref{fig:strong_tsfm_results} shows that when TSFM forecasts have a good performance, a mixture of experts method can still provide an improvement. Furthermore, \algo is competitive with its baseline methods.
\begin{figure}[ht]
    \centering
    \includegraphics[width=.8\linewidth]{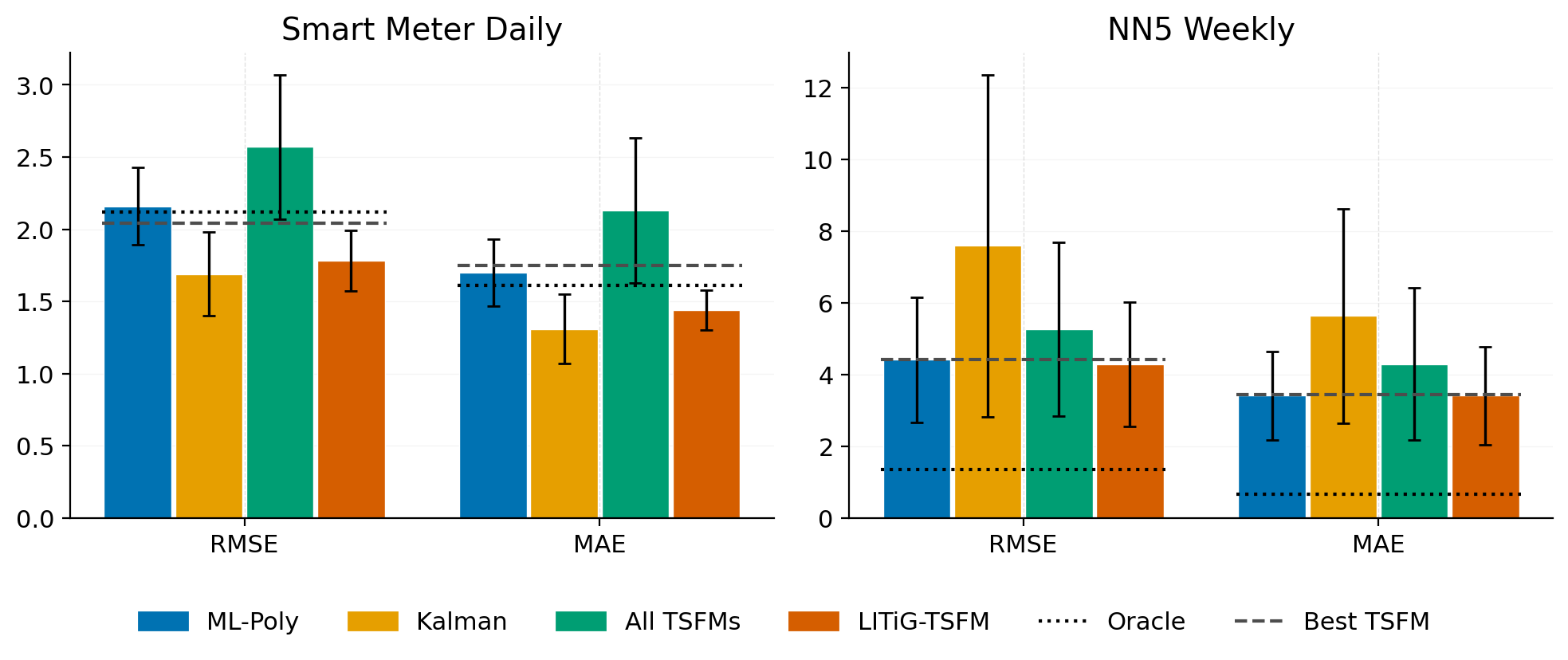}
    \caption{Average forecasting performance \(\times10^2\) on datasets with strong TSFM forecasts. Lower is better. Note that the $y$ axes are on different scales.}
    \label{fig:strong_tsfm_results}
\end{figure}

The TSFM forecasts are much weaker on the Washington Bicycle Share and Mauna Loa datasets. On the Washington Bicycle Share dataset, the experts have an average (std) RMSE of 19.45 (6.61) and an average (std) MAE of 14.70 (5.65). For the Mauna Loa dataset, the average (std) RMSE is 45.27 (17.33) and the average (std) MAE is 35.34 (13.84). Therefore, \Cref{fig:weak_tsfm_results} does not show these metrics, to compare competitive mixture of experts methods only.
For the Washington Bicycle Share dataset, the weak TSFM forecasts are due to the high frequency and levels of noise in the data. For the Mauna Loa data, the weak TSFM forecasts can be put down to the strong trend component of the data, which is monotonically increasing: the TSFM fits on the train fold of the data, and the test set is sequentially after the training fold, but the time series continues to grow beyond the domain originally seen. Whilst the TSFM can still detect the periodicity in the dataset, it will go back to the mean of the training fold, rather than correctly continue the growth in the training set. These two datasets present compelling use-cases for mixture-of-experts methods. 
\begin{figure}[ht]
    \centering
    \includegraphics[width=.8\linewidth]{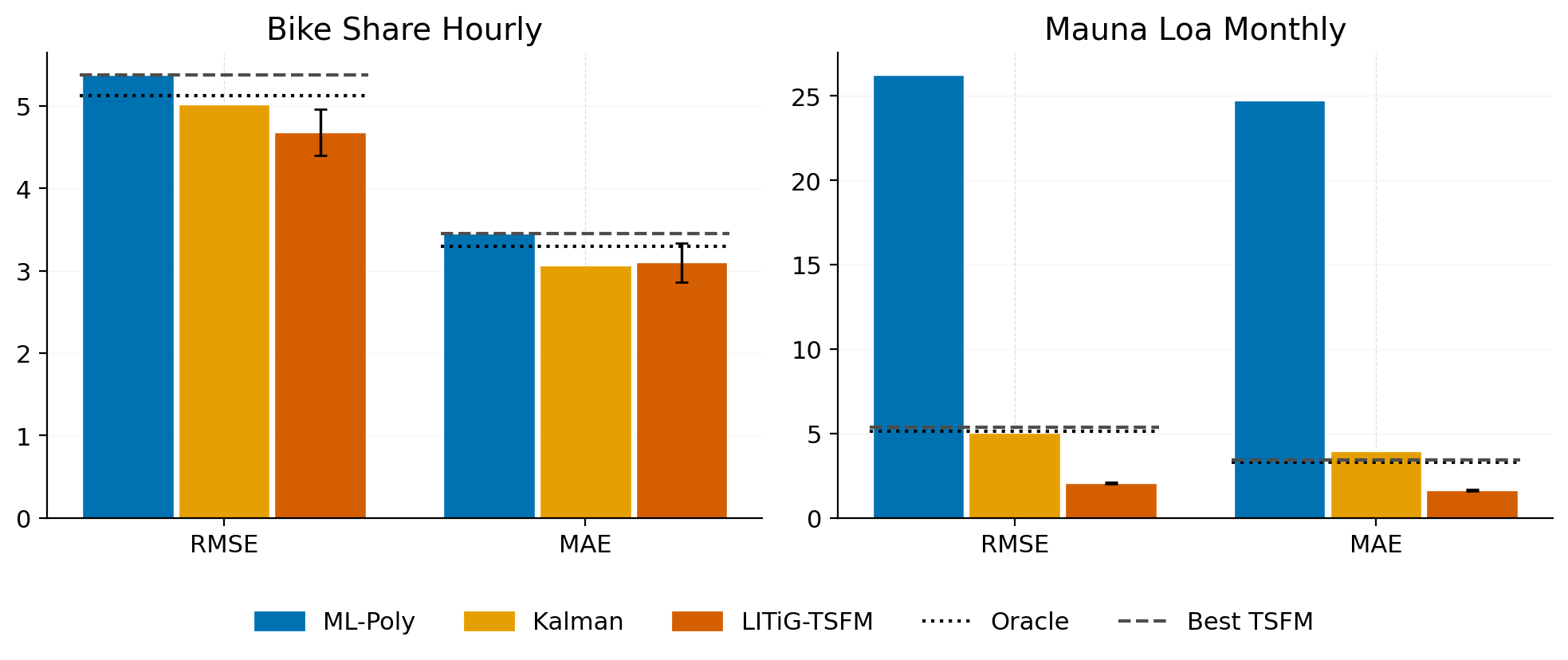}
    \caption{Average forecasting performance \(\times10^2\) on datasets with poor TSFM forecasts: for the Washington Bicycle Share dataset, mean RMSE (std) is 19.45 (6.61) and mean MAE (std) is 14.70 (5.65); for the Mauna Loa dataset, mean RMSE (std) is 45.27 (17.33) and mean MAE (std) is 35.34 (13.84). Lower is better. Note that the $y$ axes are on different scales.}
    \label{fig:weak_tsfm_results}
\end{figure}

\textbf{Expert ablation. } Due to its formulation, the ML-Poly algorithm gives a forecast contained in the convex hull given by the expert forecasts. Thus, when an oracle forecast is not particularly strong, ML-Poly can only prove to a certain extent. This motivates the use of a Kalman Filter in the first instance. \algo then intervenes to go beyond the linear setting provided by the Kalman Filter. To illustrate this, we work in a setting where the performance of the TSFMs is known. We control the quality of the experts used in the ensemble methods to study whether \algo is sensitive to poor-performing experts, and whether its performance deteriorates faster than other ensembling methods. After ranking the TSFMs' performance on the test segments, the strongest nine experts are chosen as the set of best experts; the five weakest experts are retained as the set of poor experts. The worst is taken for the setting with one poor expert, and the three worst are taken for the setting with three poor experts.  
\Cref{fig:ablation:bikeshare_experts} compares \algo to its ensembling baselines' performances when these poor experts are added to a strong set of forecasts. The error of the TSFMs grows with the addition of weaker experts; \algo stays robust to this and does not deteriorate, which makes it competitive with the ML-Poly and Kalman Filter baselines.
\begin{figure}[ht]
    \centering
    \includegraphics[width=.8\linewidth]{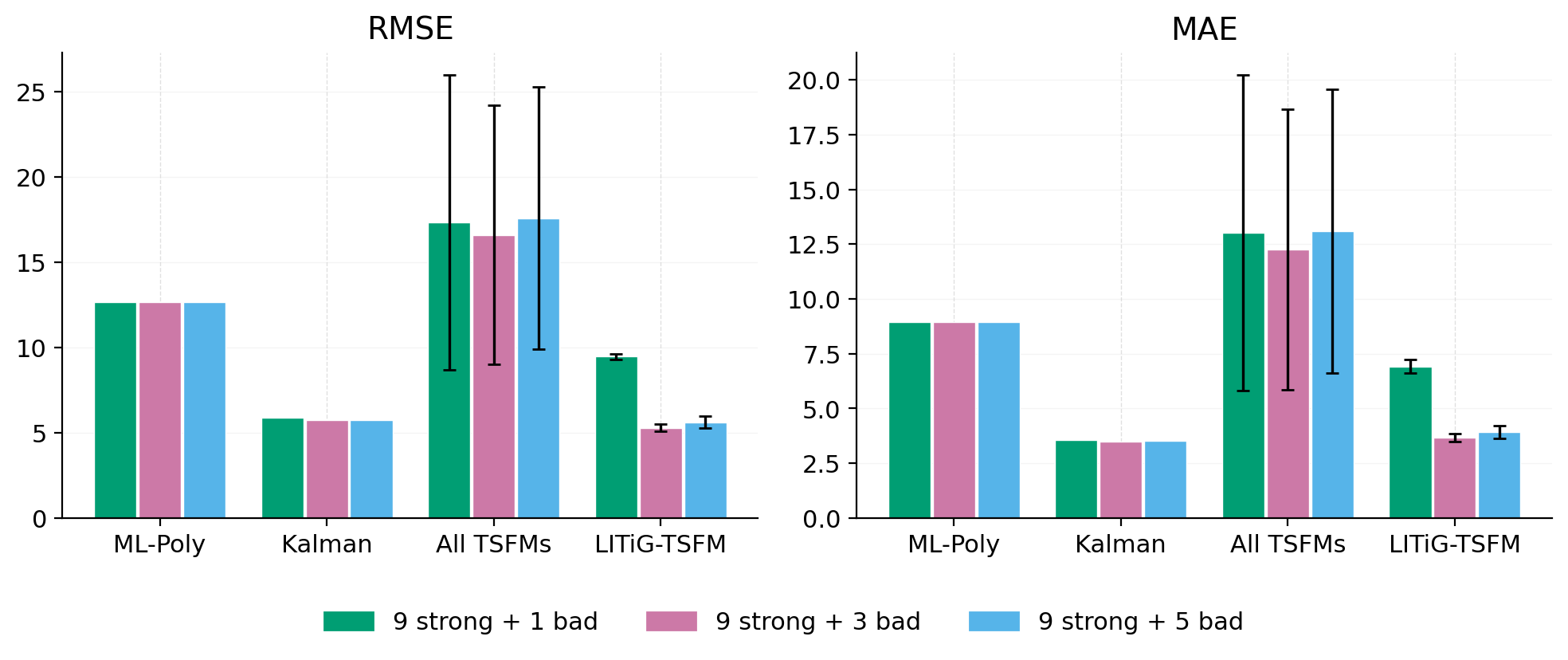}
    \caption{Experts ablation: \algo performance compared to baseline ensemble methods on the Washington Bicycle Share dataset, with nine strong experts and one, three and five poor ones respectively. Metrics given $\times10^{-2}$, lower is better.}
    \label{fig:ablation:bikeshare_experts}
\end{figure}

\Cref{app:experiments} provides a further ablation study on the architecture of the encoder and decoder, which highlights the possible gains from making relevant architectural design choices. \Cref{tab:inference_time} shows that for a set of prompted experts, running a Kalman Filter on a dataset with multiple time series or with high frequency time steps will take longer than training and prompting \algo. Furthermore, \algo does not require a warm-up window as the Kalman Filter or ML-Poly algorithms do; this means that in fewer steps at inference, the forecast becomes useful, and we do not need to discard the initial steps.

%% file: sections/discussion.tex
This work lies at the intersection of non-linear ICA, variational inference and TSFMs, providing a probabilistic inference-time guidance framework (\algo) for combining TSFM forecasts through latent structured state-space modelling. This is accompanied by existing guarantees on the latent state dynamics, and adapts reconstruction guarantees to the proposed frameworks, as well as providing novel forecasting guarantees. Whilst we use forward factorisations of the time series when implementing \algo, the reconstruction guarantee is given in a backward factorisation, as it builds on existing works. Furthermore, some of the assumptions made to achieve the results can be too restrictive for a real-life dataset. Our experiments use TabICLv2 as an example of a TSFM to create experts; however, it would be possible to implement the same experiments with any existing out-of-the-box TSFM. 

%% file: sections/app_technical.tex
\subsection{Derivation of the Evidence Lower Bound}
\label{subapp:elbo}
Recall that the target data is \(\mathbf{y}_{1:T} \in \R^{T \times m}\), with covariate data \(\mathbf{x}_{1:T} \in \R^{T \times d}\) and the \(N\) expert forecasts for \(\mathbf{y}_{1:T}\) are \(\mathbf{F}_{1:T}^i \in\R^{N \times T \times m}\), where \(1\leq i \leq N\). 
For ease of notation, denote \(\mathbf{x}_{1:T} = (\mathbf{x}_{1:T}, \mathbf{F}_{1:T}^{1:N})\) for the rest of this derivation. 
Consider the joint distribution \(p_\theta(\mathbf{z}_{1:T},\mathbf{x}_{1:T},\mathbf{y}_{1:T})\) whose parameters $\theta\in\Theta$ we would like to learn. Since the posterior \(p_{\theta}(\mathbf{z}_{1:T}\mid\mathbf{x}_{1:T}, \mathbf{y}_{1:T})\) is intractable, we use variational inference to approximate it, and maximise the ELBO on parameters \((\theta,\varphi)\), which is given by
\[
\mcl(\mathbf{y}_{1:T},\mathbf{x}_{1:T}; \theta, \varphi) = \E_{q_\varphi(\cdot\mid\mathbf{y}_{1:T},\mathbf{x}_{1:T})} \left[\log \frac{p_\theta(\cdot, \mathbf{y}_{1:T}\mid\mathbf{x}_{1:T})}{q_\varphi(\cdot\mid \mathbf{y}_{1:T},\mathbf{x}_{1:T})}\right] \eqsp.
\]
With the definition of the joint distribution and the properties of the logarithm, this can be written as:
\begin{multline}
\label{eq:additive_elbo}
    \mcl(\mathbf{y}_{1:T},\mathbf{x}_{1:T}; \theta, \varphi) \\= \E_{q_\varphi(\cdot\mid\mathbf{y}_{1:T},\mathbf{x}_{1:T})}\left[\log p_{\theta}(\cdot) +\log p_\theta(\mathbf{y}_{1:T}\mid \mathbf{x}_{1:T}, \cdot) - \log q_\varphi(\cdot\mid\mathbf{y}_{1:T}, \mathbf{x}_{1:T})\right] \eqsp.
\end{multline}

Within the expectation, each term can be derived in closed form. The \emph{prior} is factorised as \(p_\theta(\mathbf{z}_{1:T}) = p_{\theta}(z_1)\prod_{t=2}^Tp_\theta(z_{t}\mid{z}_{t-1})\), where we parametrise \(p_\theta(z_1)\) with \(\mcn(z_1; \mu_1, \rho_1)\) and \(p_\theta(z_{t}\mid{z_{t-1}})\) is parametrised by \(\mcn(z_{t+1};f_\theta(x_{t+1}, z_t), \rho^2)\). The logarithm of the prior distribution is therefore given by: 
\begin{align*}
    \log p_\theta(\mathbf{z}_{1:T}) & = \log \left( p_\theta({z}_{1}) \prod_{t=2}^T p_\theta({z}_{t}\mid{z}_{t-1}) \right) \\
    & =\log p_\theta({z}_{1}) + \sum_{t=2}^T \log p_\theta({z}_{t}\mid{z}_{t-1}) \\
    &= -\frac{1}{2} \log(2\pi\rho_1^2) - \frac{(z_1-\mu_1)^2}{2\rho_1^2} + \sum_{t=2}^T \left(-\frac{1}{2}\log(2\pi\rho^2) - \frac{(z_t - f_\theta(x_t, z_{t-1}))^2}{2\rho^2}\right) \\
    &= -\frac{1}{2} \log(2\pi\rho_1^2) - \frac{(z_1 - \mu_1)^2}{2\rho_1^2} - \frac{T}{2}\log(2\pi\rho^2) - \sum_{t=2}^T \frac{(z_t - f_\theta(x_t, z_{t-1}))^2}{2\rho^2} \\
    &= -\frac{1}{2} \left(\log(2\pi\rho_1^2) + T\log(2\pi\rho^2) - \frac{(z_1 -\mu_1)^2}{2\rho_1^2} - \sum_{t=2}^T\frac{(z_t-f_\theta(x_t, z_{t-1}))^2}{2\rho^2}\right) \eqsp.
\end{align*}

The Gaussian \emph{decoder} \(p_\theta(\mathbf{y}_{1:T}\mid \mathbf{z}_{1:T}, \mathbf{x}_{1:T})\) factorises into \(p_\theta(\mathbf{y}_{1:T}\mid \mathbf{z}_{1:T}) = \prod_{t=1}^T p_\theta(y_t\mid z_t)\), where the conditioning on \(x_t\) is implicit from the dependency on \(z_t\), and where \(p_\theta(y_t\mid z_t)\) is parametrised by \(\mcn(y_t;h_\theta(F_t,z_t), \Sigma)\). The logarithm of this distribution is given by:
\begin{align*}
    \log p_\theta (\mathbf{y}_{1:T} \mid \mathbf{z}_{1:T}) &= \log \prod_{t=1}^T p_\theta(y_t\mid z_t) \\
    &= \sum_{t=1}^T \log p_\theta(y_t\mid z_t) \\
    &= -\frac{T}{2}\log\left(\mathrm{det}(2\pi \Sigma)\right) - \frac{1}{2}\sum_{t=1}^T\left(y_t - h_\theta(z_t,F_t)\right)^\top \Sigma^{-1}\left(y_t - h_\theta(z_t,F_t)\right)\eqsp.
\end{align*}

The posterior \(p_\theta(\mathbf{z}_{1:T}\mid\mathbf{x}_{1:T}, \mathbf{y}_{1:T})\) is intractable. We estimate it with the variational family \(q_\varphi(\mathbf{z}_{1:T}\mid\mathbf{x}_{1:T}, \mathbf{y}_{1:T}) = q_\varphi(z_1\mid x_1,y_1)\prod_{t=2}^T q_\varphi(z_t\mid\mathbf{x}_{1:t-1}, \mathbf{y}_{1:t-1}, z_{t-1})\), where as per \eqref{eq:enc}, \( q_\varphi(z_1\mid x_1,y_1)\) and \(q_\varphi(z_t\mid\mathbf{x}_{1:t-1}, \mathbf{y}_{1:t-1}, z_{t-1})\) are Gaussians parametrised by \(\mcn(z_1;\mu_\varphi(y_1, x_1);\Sigma_\varphi(y_1, x_1))\) and \(\mcn(z_t;\mu_{\varphi, t}(z_{t-1}, \mathbf{x}_{1:t-1}, \mathbf{y}_{1:t-1}), \Sigma_{\varphi,t}(z_{t-1}, \mathbf{x}_{1:t-1}, \mathbf{y}_{1:t-1}))\) respectively. The logarithm of the variational family is therefore:
\begin{align*}
    \log &q_\varphi(\mathbf{z}_{1:T}\mid \mathbf{x}_{1:T}, \mathbf{y}_{1:T}) = \log \left( q_\varphi(z_1 \mid x_1, y_1)\prod_{t=2}^T q_\varphi(z_t\mid\mathbf{x}_{1:t-1}, \mathbf{y}_{1:t-1}, z_{t-1}) \right) \\
    &= \log q_\varphi(z_1 \mid x_1, y_1)+\sum_{t=2}^T \log q_\varphi(z_t\mid\mathbf{x}_{1:t-1}, \mathbf{y}_{1:t-1}, z_{t-1}) \\
    &= \log\left[\frac{1}{\sqrt{2\pi\sigma_\varphi^2(x_1)}}\exp\left(-\frac{(z_1 - \mu_\varphi(y_1))^2}{2\sigma^2_\varphi(x_1)}\right)\right] + \sum_{t=2}^T \log\left[\frac{1}{\sqrt{2\pi\sigma^2_{\varphi,t}}} \exp\left(-\frac{(z_t - \mu_{\varphi,t}(\mathbf{y}_{1:t}))^2}{2\sigma^2_{\varphi,t}}\right)\right] \\
    & =-\frac{1}{2}\left(\log(2\pi\sigma^2_\varphi(x_1)) - \frac{(z_1-\mu_\varphi(y_1))^2}{2\sigma^2_{\varphi,t}(x_1)}\right) + \sum_{t=2}^T\left[\log(2\pi\sigma^2_{\varphi,t}) - \frac{(z_t - \mu_{\varphi,t}(\mathbf{y}_{1:t}))^2}{2\sigma^2_{\varphi, t}}\right] \eqsp.
\end{align*}
Plugging these values into \eqref{eq:additive_elbo} gives the following ELBO:
\begin{align*}
    \mathcal{L}(\theta,\phi; \mathbf{y}_{1:T}, \mathbf{x}_{1:T}) &= \E_{q_\varphi(\cdot\mid\mathbf{y}_{1:T}, \mathbf{x}_{1:T})} \biggl[-\frac{1}{2}\log(2\pi\rho_1^2) + T \log(2\pi\rho^2) - \frac{(z_1-\mu)^2}{2\rho_1^2} \\
    &- \sum_{t=2}^T\frac{(z_t-f_\theta(x_t, z_{t-1}))^2}{2\rho^2} + \frac{1}{2}\biggl(-T\log(2\pi\sigma^2)+\sum_{t=1}^T\left(-\frac{(y_t - h_\theta(z_t,F_t))^2}{2\sigma^2}\right)\\
    &+\frac{1}{2}\log(2\pi\rho_1^2)-T\log(2\pi\rho^2) + \frac{(z_1 -\mu)^2}{2\rho_1^2} + \sum_{t=1}^T\frac{(z_t - f_\theta(x_t, z_{t-1}))^2}{2\rho^2}\biggr)\biggr] \eqsp.
\end{align*}

\subsection{Derivation of the forecasting distribution}
\label{subapp:forecasting}
The conditional expectation of a new observation given past observations and covariates can be written as:
\begin{align*}
    \E_{p_\theta(\cdot\mid \mathbf{y}_{1:t},\mathbf{x}_{1:t+1})}\left[y_{t+1}\right] &= \int 
    p_\theta(y_{t+1}\mid \mathbf{y}_{1:t},\mathbf{x}_{1:t+1})\rmd y _{t+1} \\
    &= \int
    p_\theta(y_{t+1}\mid z_{t+1}) p_\theta(z_{t+1}\mid z_{t},x_{t+1}) p_\theta(z_t\mid\mathbf{y}_{1:t},\mathbf{x}_{1:t}) \rmd y_{t+1} \rmd z_{t+1} \rmd z_t\eqsp.
\end{align*}
This predictive expectation can be approximated using the trained variational family:
\[
q_\varphi(\mathbf{z}_{1:T} \mid \mathbf{y}_{1:T}, \mathbf{x}_{1:T})
= q_{\varphi,T}(z_T \mid \mathbf{y}_{1:T}, \mathbf{x}_{1:T})\prod_{t=1}^{T-1} q_{\varphi,t|t+1}(z_{t} \mid z_{t+1}, \mathbf{y}_{1:t}, \mathbf{x}_{1:t})\eqsp.
\]
In this setting, the expectation under $p_\theta(z_t\mid\mathbf{y}_{1:t},\mathbf{x}_{1:t})$ can be approximated either by an expectation under $q_\varphi({z}_{t} \mid \mathbf{y}_{1:t}, \mathbf{x}_{1:t})$ or using a VAMP-like sampler \citep{tomczak_vae_2018}:
\[
\mu_{\mathrm{vamp}}(z_t) = \frac{1}{n}\sum_{i=1}^n q_\varphi({z}_{t}\mid\mathbf{y}_{1:t}^i,\mathbf{x}_{1:t})\eqsp.
\]

\subsection{Proof of Proposition~\ref{prop:reconstruction}}
\label{subapp:pf_reconstruction}
\begin{proof}
Conditionally on $(\mathbf{x}_{1:T},\mathbf{F}_{1:T})$ the proposed nonlinear ICA model is a general state space model. Under the strong mixing assumption {\bf B2}, we can apply Proposition~1 of \cite{JMLR:v25:22-1392} using the additive state functional
$$
h_{1:T}(\mathbf{z}_{1:T}) = \sum_{t=1}^T\ell_t(z_t) = \sum_{t=1}^T\|y_t-h_\theta(\mathbf{F}_t,z_t)\|^2\eqsp.
$$
Therefore, for all $1\leq k \leq T-1$, and all probability densities $\tilde q_k$,
\begin{multline*}
    \left|
        \sum_{t=1}^T
        \mathbb E_{q_{\varphi}(\cdot | \mathbf{y}_{1:T}, \mathbf{x}_{1:T})}[\ell_t(z_t)]
        -
        \sum_{t=1}^T
        \mathbb E_{\phi^\theta_{1:T}(\cdot | \mathbf{y}_{1:T}, \mathbf{x}_{1:T})}[\ell_t(z_t)]
    \right| \leq 2\frac{\sigma_+}{\sigma_-}\sum_{t=1}^{T-1}\|\ell_t\|_{\infty}\\
    \times\left(c_1(\theta) + \sum_{m=2}^k\rho^{k-m+1}c_m(\theta,\varphi) + c_{k+1}(\theta,\varphi)+  \sum_{m=k+2}^{T}\rho^{m-k-1}c_m(\theta,\varphi)\right)\eqsp,
\end{multline*}
where $\rho = 1 - \sigma_-/\sigma_+$ and $c_1(\theta) = \|\tilde q_1 - \phi^\theta_{1}\|_{\mathrm{tv}}$ and for $2\leq k \leq T$,
$$
c_k(\theta,\varphi) = \left \|\tilde \phi^\theta_{k-1|k} - \tilde \nu^\varphi_{k-1|k} \right\|_{\mathrm{tv}}\eqsp,
$$
where $\tilde \phi^\theta_{k-1|k}(z_{k-1},z_k) \propto \tilde q_{k-1}(z_{k-1})p_\theta(z_{k}|z_{k-1},x_{k})p_\theta(y_{k}|z_{k})$ and $\tilde \nu^\varphi_{k-1|k}(z_{k-1},z_k) = \tilde q_k(z_k)q_{\varphi,k-1|k}(z_{k-1}\mid z_{k}, \mathbf{x}_{1:k-1}, \mathbf{y}_{1:k-1})$.
By choosing for all $1\leq k \leq T$, $\tilde q_k = \phi^\theta_{k}$, yields $c_1(\theta) = 0$ and, dropping the dependency on the covariates $\mathbf{x}_{1:T}$ for better clarity,
\begin{align*}
    c_k(\theta,\varphi) & = \left \|\frac{\phi^\theta_{k-1}(z_{k-1})p_\theta(z_{k}|z_{k-1})p_\theta(y_{k}|z_{k})}{\int \phi^\theta_{k-1}(z_{k-1})p_\theta(z_{k}|z_{k-1})p_\theta(y_{k}|z_{k}) \rmd z_{k-1}\rmd z_k} -\phi^\theta_{k}(z_k)q_{\varphi,k-1|k}(z_{k-1}\mid z_{k}, \mathbf{y}_{1:k-1}) \right\|_{\mathrm{tv}}\\
    &\leq \left \|\frac{\phi^\theta_{k-1}(z_{k-1})p_\theta(z_{k}|z_{k-1})p_\theta(y_{k}|z_{k})}{\int \phi^\theta_{k-1}(z_{k-1})p_\theta(z_{k}|z_{k-1})p_\theta(y_{k}|z_{k}) \rmd z_{k-1}\rmd z_k} -\phi^\theta_{k}(z_k)b^\theta_{k-1|k}(z_{k-1}\mid z_{k},  \mathbf{y}_{1:k}) \right\|_{\mathrm{tv}} \\
    &\hspace{2cm}+\left \|\phi^\theta_{k}(z_k)b^\theta_{k-1|k}(z_{k-1}\mid z_{k},  \mathbf{y}_{1:k}) -\phi^\theta_{k}(z_k)q_{\varphi,k-1|k}(z_{k-1}\mid z_{k}, \mathbf{y}_{1:k-1}) \right\|_{\mathrm{tv}}\eqsp.
\end{align*}
By definition of the backward kernel, the first term in the last inequality is 0 and therefore, by assumption {\bf B3},
$$
c_k(\theta,\varphi) \leq \varepsilon\eqsp.
$$
Then, by assumption {\bf B1} which gives \(\|y_t - h_\theta(\mathbf{F}_t, z_t)\|^2 \leq \|y_t\|+\|h_\theta(\mathbf{F}_t,z_t)\|\leq c_2 + c_1\), we can bound the square norm:
$$
\|y_t-h_\theta(\mathbf{F}_t,z_t)\|^2\leq 2c_2^2 + 2c_1^2\eqsp,
$$
so that $\|\ell_t\|_{\infty}\leq 2c_1^2+2c_2^2$, which gives \(C \coloneqq 2\frac{\sigma_-}{\sigma_+}(2c_2^2 + 2c_1^2)\). This concludes the proof.
\end{proof}

%% file: sections/app_experiments.tex
\paragraph{Metrics.} For a time series $\mathbf{{y}}_{1:T}$ and a sequence of predictions $\mathbf{\hat{y}}_{1:T}$, the \emph{Root Mean Square Error (RMSE)} and \emph{Mean Absolute Error (MAE)} are  given by 
\begin{equation*}
    \textrm{RMSE}(\mathbf{{y}}_{1:T}, \mathbf{\hat{y}}_{1:T}) = \left(\frac{1}{T}\sum_{t=1}^T (\hat{y}_t - y_t)^2\right)^{1/2}\eqsp,\quad
    \textrm{MAE}(\mathbf{{y}}_{1:T},\mathbf{\hat{y}}_{1:T}) = \frac{1}{T}\sum_{t=1}^T \lvert\hat{y}_t - y_t\rvert\eqsp.
\end{equation*}

\paragraph{TSFM prompts for expert creation.} We use TabICLv2\footnote{\url{https://tabicl.readthedocs.io/en/latest/}} \citep{qu_tabiclv2_2026} in regression mode to predict the next time step. For this appendix, we call the \emph{training fold} of the dataset the initial part of the time series which is given as context to the TSFM. The TSFM then makes single-step forecasts using that original context, and uses those forecasts for the following steps. At a chosen frequency, the ground truth for the forecast made replaces the forecast in the context for the following time steps. So for example, for a daily dataset with monthly prompting, the TSFM will use the original training fold, then make forecasts on a daily basis for a month; at the end of that month, the true value of the time series will be revealed and added to the context, replacing the last month's forecasts. For the NN5, Mauna Loa and Washington Bicycle Share datasets, we use a high-frequency prompt and a lower-frequency prompt. For the London Smart Meter dataset, we only use low-frequency prompts. For each frequency used, we use both the mean and the raw quantile outputs from TabICLv2 as our experts.

\paragraph{Datasets.}
\emph{Individual consumption aggregated by socio-economic profile. } The London Smart Meter dataset \citep{SmartMeter_dataset} records the electricity consumption of households with a smart meter in London between 2012 and 2014. The covariates include Acorn segments, which provide a socio-demographic segmentation\footnote{\url{https://acorn.caci.co.uk/how-acorn-works/}} for each household. The dataset is aggregated by Acorn segments: experiments are run on a group of these segment (segments K, L, M, N and O, which are part of the `Steadfast Communities' group), to see how \algo handles energy data in a noisy setting. It also evaluates how the model performs when learning to forecast for multiple categories at once. Electricity consumption is the target variable; more information on the available covariates and how they are preprocessed is given in \Cref{tab:london_smart_meter}.
\begin{table}[ht]
\centering
\begin{tabular}{c|c}
\textbf{Feature}     & \textbf{Preprocessing}     \\ \hline \hline
Acorn segment        & One-hot encoded            \\ \hline
Number of clients    & Normalised between 0 and 1 \\ \hline
Date                 & Year cyclically encoded    \\ \hline
Weekday              & One-hot encoded            \\ \hline
Temperature          & Normalised between 0 and 1 \\ \hline
Consumption (target) & Normalised between 0 and 1
\end{tabular}
\caption{Preprocessing for variables in the London Smart meter dataset.}
\label{tab:london_smart_meter}
\end{table}

\emph{UK ATM demand. } The NN5 dataset \citep{godahewa2020nn5} provides weekly demand for 111 ATMs in the UK. Besides the demand (which is the target variable), only the date is known -- as noted in \Cref{tab:nn5_dataset} -- the month and week of the year are cyclically encoded. This tests \algo in a setting with multiple series learned at once, despite having very little covariate information and a large amount of noise.
\begin{table}[ht]
\centering
\begin{tabular}{c|c}
\textbf{Feature} & \textbf{Preprocessing}     \\ \hline\hline
Demand (target)  & Normalised between 0 and 1 \\ \hline
Month of year    & Cyclically encoded         \\ \hline
Week of year     & Cyclically encoded        
\end{tabular}
\caption{Preprocessing for variables in the NN5 weekly dataset.}
\label{tab:nn5_dataset}
\end{table}

\emph{Atmospheric CO2 levels. } \citet{maunaloa_dataset} provide monthly mean measurements of atmospheric CO$_{2}$ at Mauna Loa. It has strong seasonal and trend components. The target data is detrended by differencing (each \(y_t\) is replaced by \(y_t^{\textrm{detrended}} = y_t -  y_{t-1}\)), to avoid testing on out-of-domain data. Preprocessing information for available covariates is given in \Cref{tab:mauna_loa_dataset}. This datasets tests \algo in a setting where the distribution of the test fold is different from that of the training fold. 
\begin{table}[ht]
\centering
\begin{tabular}{c|c}
\textbf{Feature}                             & \textbf{Preprocessing}             \\ \hline\hline
Positional encoding                          & Encoded between 0 and 1            \\ \hline
Year                                         & Normalised between 0 and 1         \\ \hline
Month                                        & Cyclically encoded                 \\ \hline
Mean monthly concentration (target)       & Detrended, normalised between 0 and 1; lagged \\ \hline
Deseasonalised mean monthly concentration & Detrended, normalised between 0 and 1; lagged \\ \hline
Number of days with measurements             & Normalised between 0 and 1         \\ \hline
Standard deviation                           & Normalised between 0 and 1         \\ \hline
Uncertainty                                  & None                              
\end{tabular}
\caption{Preprocessing for variables in the Mauna Loa monthly dataset.}
\label{tab:mauna_loa_dataset}
\end{table}

\emph{Bicycle share dataset. } The Washington Bicycle Share dataset \citep{Bike_dataset} provides the number of rented bicycles on an hourly time step over 2011 -- 2012, alongside weather and calendar covariates. Two categories of users exist: subscribed users and casual users. The total number of users is given as a target variable. The lags of subscribed users and casual users are given as covariates, as well as the lag for the total number of users. This dataset tests \algo on a noisy, high-frequence dataset. Further information on how these variables were preprocessed is given in \Cref{tab:bike_share_dataset}.
\begin{table}[ht]
\centering
\begin{tabular}{c|c}
\textbf{Feature}                 & \textbf{Preprocessing}                \\ \hline\hline
Season                           & One-hot encoded                       \\ \hline
Bank holiday                     & One-hot encoded                       \\ \hline
Day of week                      & One-hot encoded                       \\ \hline
Hour                             & Cyclically encoded                    \\ \hline
Month                            & Cyclically encoded                    \\ \hline
Working day                      & One-hot encoded                       \\ \hline
Temperature                      & Normalised between 0 and 1            \\ \hline
Average temperature              & Normalised between 0 and 1            \\ \hline
Humidity                         & Normalised between 0 and 1            \\ \hline
Windspeed                        & Normalised between 0 and 1            \\ \hline
Weather conditions               & One-hot encoded                       \\ \hline
Number of non-subscription users & Normalised between 0 and 1, lag taken \\ \hline
Number of subscribed users       & Normalised between 0 and 1, lag taken \\ \hline
Total users (target)             & Normalised between 0 and 1, lag added
\end{tabular}
\caption{Preprocessing for variables in the Washington Bicycle Share dataset.}
\label{tab:bike_share_dataset}
\end{table}

\paragraph{Model architecture. } \algo is implemented in PyTorch. Training is conducted with the AdamW optimiser \citep{loshchilov2018decoupled} that uses a cyclic learning rate scheduler \citep{smith2017cyclical}. The components of the \algo model are implemented as follows: 
\begin{itemize}
    \item \emph{Encoder.} An attention layer \citep{vaswani2017attention} is preceded and followed by a normalisation layer \citep{ba2016layernormalization}. This is followed by two linear layers, separated by a GELU activation function. The output mean has a sigmoid activation function; the output variance is clamped between -6 and 0.
    \item \emph{Prior.} The prior is made up of a linear input layer, a GELU activation function and an output layer which feeds through a sigmoid activation function.
    \item \emph{Decoder.} The decoder is made up of a linear input layer, a GELU activation function, a linear output layer with a sigmoid activation function.
\end{itemize}

\paragraph{Hyperparameters. }
Each dataset requires domain- and frequency-aware hyperparameter selection. These modelling choices enable strong forecasting performance. \Cref{tab:hyperparameters} gives the main hyperparameters which required per-dataset tuning.
\begin{table}[ht]
\centering
\begin{tabular}{c|c|c|c|c}
\textbf{Hyperparameter}   & \textbf{Smart Meter} & \textbf{Bike Share} & \textbf{NN5 Weekly} & \textbf{Mauna Loa} \\ \hline \hline
Sequence length           & 64                   & 96                  & 16                  & 196                 \\ \hline
Batch size               & 128                  & 256                 & 64                  & 128                 \\ \hline
Learning rate             & $5\times10^{-3}$     & $1\times10^{-3}$    & $1\times10^{-3}$    & $5\times10^{-3}$    \\ \hline
Latent dimension          & 64                   & 64                  & 32                  & 64                  \\ \hline
Epochs                    & 200                  & 20                  & 60                  & 500                 \\ \hline
Number heads in encoder   & 4                    & 4                   & 4                   & 4                   \\ \hline
Encoder hidden layer      & 128                  & 256, 128            & 128                 & 128                 \\ \hline
Prior hidden layer        & 64                   & 64                  & 64                  & 64                 \\ \hline
Decoder hidden layer      & 128                  & 128                 & 128                 & 64                       
\end{tabular}
\caption{Hyperparameters chosen for each dataset.}
\label{tab:hyperparameters}
\end{table}

\paragraph{Inference-time compute.} In our setting, implementing the ML-Poly mixture of experts method requires almost no additional computational overhead. Calculating the oracle forecast and the best expert likewise does not require additional time. \Cref{tab:inference_time} provides compute times for the TSFM inference (on higher and lower frequency of prompts, for mean and quantile outputs respectively), the expectation maximisation for the Kalman Filter and training and inference time for the \algo algorithm. The additional compute overhead needed to run expectation maximisation on the Kalman Filter further justifies using a non-linear encoder for \algo: as stated in \Cref{sec:theory}, the linear latent state cannot be disentangled. This is especially true for the London Smart Meter and NN5 datasets, where the Kalman Filter needs to be refit for each time series in the dataset, whereas \algo only requires to be trained once and can then run inference for every series in the dataset.
\begin{table}[ht]
\centering
\begin{tabular}{c|c|c|c|c}
\textbf{Dataset} & \textbf{TSFM low freq.} & \textbf{TSFM high freq.} & \textbf{Kalman} & \textbf{\algo train + infer} \\ \hline\hline
Smart Meter*     & 1'44, 1'27              & Not used                 & 35”             & 3’04” + 1”                   \\ \hline
Bike Share       & 39', 10'                & 60', 10'                 & 13'02"          & 9'48" + 33"                  \\ \hline
NN5*             & 10", 9"                 & 40", 59"                 & 5"              & 37" + 1"                     \\ \hline
Mauna Loa        & 1’, 1’                  & 14’, 16’                 & 1’              & 4’ + 2”                          
\end{tabular}
\caption{Compute cost for \algo compared to the baseline models. Datasets with a * show cost for fitting the Kalman Filter and running \algo inference a single series. TSFM prompting times show the time needed to make a mean forecast and a raw quantiles forecast.}
\label{tab:inference_time}
\end{table}

\paragraph{Qualitative samples. }
We provide sample time series and their forecasts for the éCO$_2$mix, London Smart Meter (Acorn segment K) and Washington Bicycle share datasets. This provides an additional qualitative perspective on where a mixture of experts method is a valuable addition to TSFM forecasts. 

\Cref{fig:london_sample} shows a sample from the London Smart Meter dataset, which provides a challenging case to the experts, where they consistent predict smaller quantities than the actual outcome. \algo is able to correct for this, and has more occurrences of clearly identifying the extreme values of peaks and valleys than the Kalman Filter. It also it has the advantage of only needing to be trained once for all the segments in the time series; whereas the Kalman Filter would require refitting for each time series in the dataset.
\begin{figure}[ht]
    \centering
    \includegraphics[width=.75\linewidth]{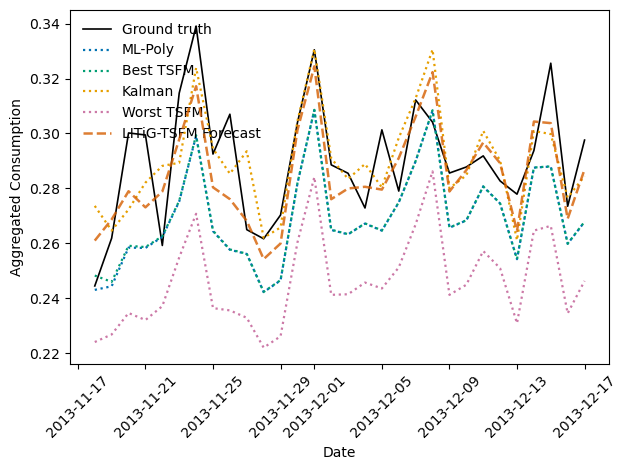}
    \caption{London Smart Meter, Acorn Segment K forecast sample of \algo and its baselines.}
    \label{fig:london_sample}
\end{figure}

The sample from the Washington Bicycle Share dataset in \Cref{fig:bike_share_sample} is similar to the London Smart Meter segment shown, where there is a large gap between the worst performing expert and target value. \algo is again capable of fitting the peaks and valleys, and it is faster at training and inference combined, compared to the Kalman Filter. This is due to the length of the time series to forecast, which is at much higher frequency than the London Smart Meter sample provided above.
\begin{figure}[ht]
    \centering
    \includegraphics[width=.75\linewidth]{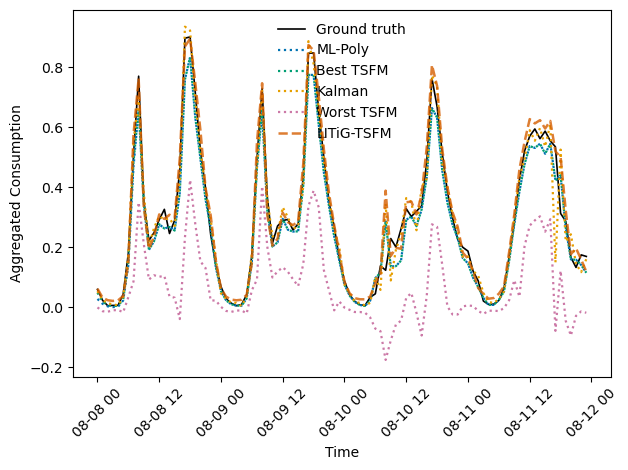}
    \caption{Washington Bike Share forecast sample of \algo and its baselines.}
    \label{fig:bike_share_sample}
\end{figure}

\Cref{fig:nn5_sample} provides a sample for a sample from the NN5 dataset, alongsides the forecasts from the baseline approaches. This shows that ML-Poly and Kalman either other-smooth the mixture of predictions, or are not able to learn the pattern consistently. \algo bridges these two approaches.
\begin{figure}[ht]
    \centering
    \includegraphics[width=.75\linewidth]{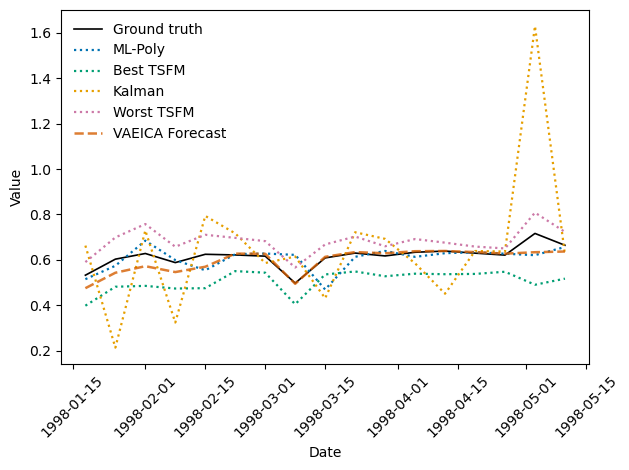}
    \caption{NN5 forecast sample of \algo and its baselines.}
    \label{fig:nn5_sample}
\end{figure}

\Cref{fig:mauna_loa_sample} provides a forecast from the detrended Mauna Loa time series.
\begin{figure}[ht]
    \centering
    \includegraphics[width=.75\linewidth]{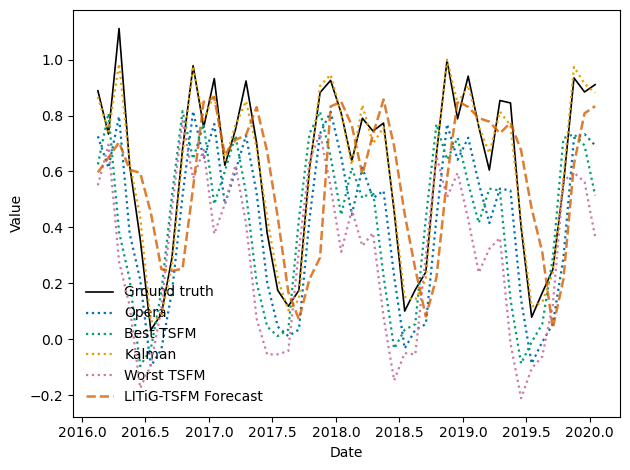}
    \caption{Detrended Mauna Loa forecast sample of \algo and its baselines.}
    \label{fig:mauna_loa_sample}
\end{figure}

\paragraph{Architecture ablation. } We conduct an ablation study on \algo's reconstruction performance, by working with different encoder and decoder architectures for the NN5 dataset. The trained models are evaluated on their reconstruction performance on the training data. The performance for each of these models is given in \Cref{fig:ablation:architecture_london}, where the MLP Decoder and Attention Encoder are as described in \Cref{app:experiments}. The \emph{MLP Encoder} is made up of a linear layer of size 32 and positional encoding of the position in the sequence, a LeakyReLU activation with slope 1.15, a linear layer of size 16 and a linear output for mean and covariance, where the mean has a Sigmoid activation layer and the variance is clamped between -6 and 2. The \emph{dot product Decoder} constrains the latent space to be of the same size as the number of experts used, as it simply takes the dot product of the latent state \(z_{t}\) by the experts \(\mathbf{F}_t\), with Gaussian noise \(\varepsilon\sim\mcn(0,1)\): \(\mathbf{y_{t}} = \mathbf{F}_t\cdot z_t + \varepsilon\). These experiments use a latent space of dimension 100, which corresponds to the number of experts: this is a constraint when using the dot product Decoder. The attention Encoder and MLP Decoder use a latent space of dimension 64.
\begin{figure}[ht]
    \centering
    \includegraphics[width=.75\linewidth]{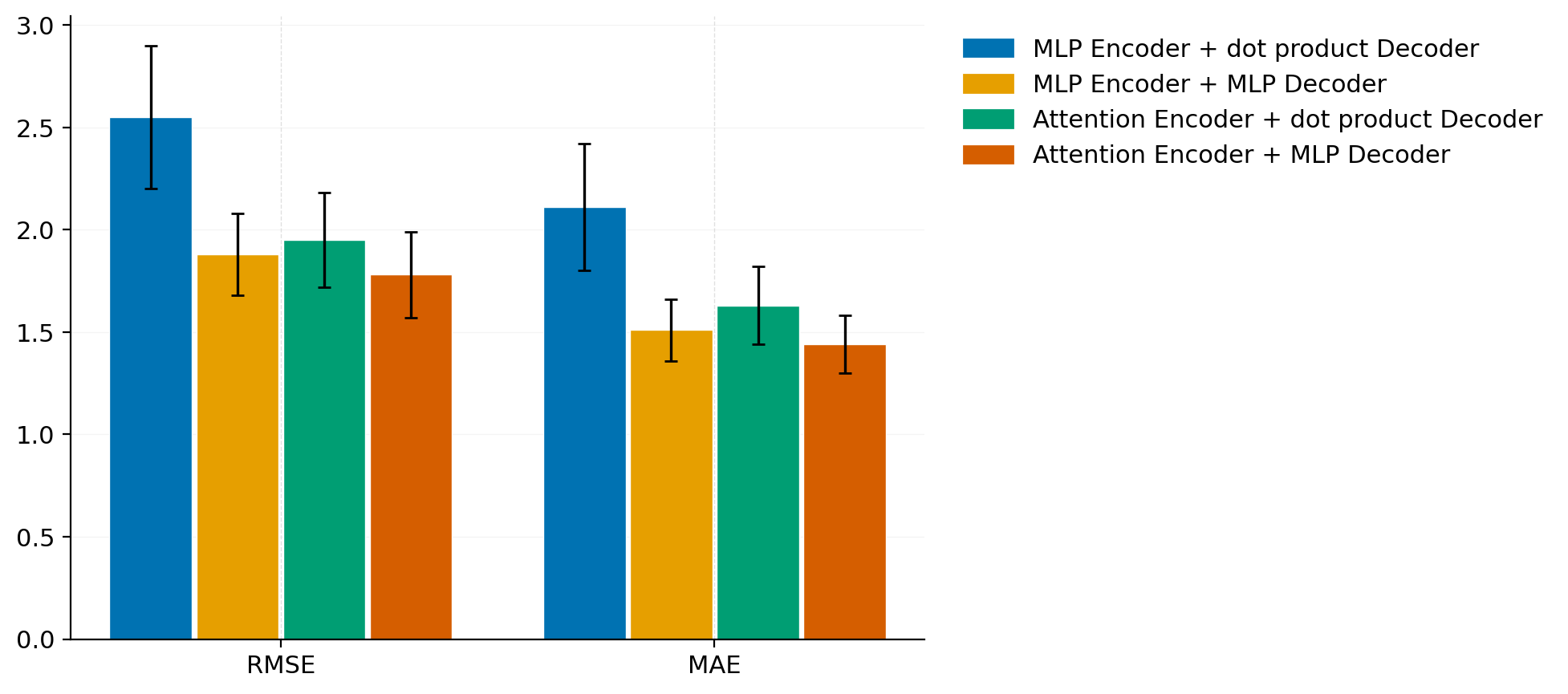}
    \caption{Architecture ablation: \algo performance on the London Smart Meter dataset with different encoder and decoders. Metrics given $\times10^{-2}$, lower is better.}
    \label{fig:ablation:architecture_london}
\end{figure}

Using a Multi-Layered Perceptron (MLP) in the decoder enables a stronger performance. Similarly, an attention-based encoder allows to better embed the sequential dependencies in the dataset. This is due to two factors: first, the most sophisticated decoder architecture means that the decoding can be learned more efficiently. Second, it means that the dimension of the latent space no longer depends on the number of experts, which gives choice for the dimension of the latent space. This gives more possibilities for model tuning, thus making it stronger.